\documentclass[conference]{IEEEtran}
\usepackage{flushend}
\usepackage{times}
\usepackage{amsmath, amsthm}
\usepackage{amsfonts}
\usepackage{mathrsfs}
\usepackage{latexsym}
\usepackage{amssymb}
\usepackage{enumerate}
\usepackage[left=1.62cm,right=1.62cm,top=1.9cm]{geometry}
\usepackage{xcolor}
\usepackage{upref}
\usepackage{graphicx}
\usepackage{color}
\usepackage{scrextend}
\usepackage[ruled]{algorithm2e} 
\usepackage{algorithmicx}
\usepackage{algpseudocode}
\usepackage{verbatim}
\usepackage{multicol}
\usepackage{colortbl}
\usepackage{multirow}
\newtheorem{theorem}{Theorem}
\usepackage{makecell}

\usepackage[bookmarks=false]{hyperref}
\usepackage{mathrsfs}
\definecolor{mygray}{gray}{.9}
\usepackage{xcolor}
\usepackage[backend=bibtex,style=ieee]{biblatex}
\usepackage{booktabs}
\usepackage{textcomp}
\usepackage{epstopdf}
\usepackage{tikz}
\usetikzlibrary{shapes ,arrows , positioning }
\theoremstyle{definition}

\newtheorem{rem}{Remark}


\begin{document}
\title{Load Balancing in Federated Learning}

 \author{
   \IEEEauthorblockN{Alireza Javani, and Zhiying Wang}
   \IEEEauthorblockA{
                     Center for Pervasive Communications and Computing\\
  					University of California, Irvine\\
                     \{ajavani, zhiying\}@uci.edu
                     \vspace{-0.8cm}}
 
 }

\maketitle

\begin{abstract}
Federated Learning (FL) is a decentralized machine learning framework that enables learning from data distributed across multiple remote devices, enhancing communication efficiency and data privacy. Due to limited communication resources, a scheduling policy is often applied to select a subset of devices for participation in each FL round. The scheduling process confronts significant challenges due to the need for fair workload distribution, efficient resource utilization, scalability in environments with numerous edge devices, and statistically heterogeneous data across devices. This paper proposes a load metric for scheduling policies based on the Age of Information and addresses the above challenges by minimizing the load metric variance across the clients. Furthermore, a decentralized Markov scheduling policy is presented, that ensures a balanced workload distribution while eliminating the management overhead irrespective of the network size due to independent client decision-making. We establish the optimal parameters of the Markov chain model and validate our approach through simulations. The results demonstrate that reducing the load metric variance not only promotes fairness and improves operational efficiency, but also enhances the convergence rate of the learning models. 

\end{abstract}

\vspace{-0.2cm}
\section{Introduction}
\vspace{-0.2cm}
Federated learning (FL), introduced by McMahan et al. \cite{mcmahan2017communication}, emerged as a solution to the limitations of traditional machine learning models that require centralized data collection and processing. FL enables client devices to collaboratively train a global model while keeping all the training data localized, thus addressing privacy concerns. 
%
%
The iterative training process of FL involves clients updating local models with their data before sending these updates to a central server for aggregation. 
To reduce communication load, a client selection policy can be applied such that only a subset of the clients transmits to the central server in a round. For example, random selection \cite{li2019convergence} selects each client uniformly at random in each round. However, client selection may lead to an unbalanced communication load and uneven contribution to the global model among clients. 

In this paper, we propose the load metric $X$, defined as the number of rounds between subsequent selections of a client. Assuming the distribution of the load metric for all clients is identical, our goal is to minimize $\operatorname{Var}[X]$, the variance of $X$.

The load metric $X$ is related to the concept of Age of Information (AoI) \cite{kaul2012real}, which is the time elapsed since the last transmission of a client. In fact, $X$ is called the peak age \cite{costa2014age}. While most previous approaches in studying AoI in networking problems involve minimizing the {average} age or the average peak age \cite{yates2021age}, this paper focuses on reducing the variance of $X$ given the probability of selecting each client. 

The benefits of minimizing the variance of $X$ are as follows:
\begin{itemize}
    \item \textbf{Predictable Update Intervals and Load Balancing:} By reducing the variance of $X$, we stabilize peak age and ensure consistent intervals between client selections. This predictability implies a fair and balanced selection process and improves operational efficiency.
    \item \textbf{Improved Convergence and Accuracy:} A lower variability in \(X\) is essential for maintaining a diverse and representative global model, preventing biases and enhancing overall model accuracy.
\end{itemize}


We introduce a decentralized policy based on a Markov chain. Here, the state of the Markov chain is determined by the age of the client, and the transition probabilities represent an age-dependent selection policy.  
We establish the optimal transition probabilities of the Markov model that minimize $\operatorname{Var}[X]$.
We demonstrate that the optimal Markov model is equivalent to selecting the oldest age. The Markov model is suitable for decentralized client selection, alleviating the management overhead of the central server. Furthermore, the Markov model can be potentially modified to allow for real-time adjustment of client selection probabilities, to accommodate the dynamic nature of client data quality, computation resources, and network configurations and conditions. 

Simulation results demonstrate that the application of the Markov model enhances the convergence rate of the learning processes, achieving target accuracies in fewer communication rounds compared to random selection, with improvements exceeding $9\%$ and more than $20\%$ in certain datasets. These findings confirm the efficacy of our approach in addressing issues related to client overloading and under-utilization while maintaining efficiency throughout the learning process.

{\bf Related Work.}
The advancements in improving the Federated Averaging (FedAvg) algorithm by McMahan et al. \cite{mcmahan2017communication} in terms of partial participation include exploring system diversity and non-IID data distribution.
Importance sampling has been employed, focusing on gradient computations on the most informative data, thereby enhancing learning outcomes \cite{balakrishnan2021resource}. Complementing this, various client selection strategies have been developed. A loss-based sampling policy \cite{goetz2019active} targets clients with high loss values to quicken convergence, though it requires careful hyper-parameter tuning. Alternatively, a sample size-based sampling policy \cite{fraboni2021clustered} opts for clients with extensive local datasets, which, while effective in accelerating convergence, can falter in non-IID scenarios.

Addressing further challenges such as data staleness and the efficient use of network resources, the Age of Information metric has been adopted \cite{yang2020age}. By prioritizing updates based on their timeliness, AoI-centric approaches enhance both accuracy and convergence rates, proving particularly effective in dynamic FL environments. In \cite{wang2022age,10279272,10255731,khan2023value}, authors use the AoI metric to determine device selection and resource allocation in federated learning over wireless networks. Their approach manages sub-channel assignments and power allocation. To the best of our knowledge, this paper is the first to address the load balancing issue based on Age of Information.

{\bf Notations.} Random variables are denoted by capital letters. $E[\cdot]$ and $\operatorname{Var}[\cdot]$ denote the expectation and variance of a random variable, respectively. 


\vspace{-0.2cm} 
\section{Problem Setting}\label{sec:problem}
\vspace{-0.2cm}

This section briefly reviews FedAvg algorithm \cite{mcmahan2017communication}. The load metric $X$ is defined and the optimization problem on $\operatorname{Var}[X]$ is formulated afterward.

{\bf FedAvg algorithm.} Consider a group of \(n\) clients coordinated by a central server to collaboratively train a common neural model using federated learning (FL). Each participating client \(i\), \(i \in \{1,2,\dots,n\}\), has a dataset \(D_i = \{X_i, Y_i\}\), where \(X_i\) is the input vector and \(Y_i\) is the corresponding output vector. These datasets are used to parameterize local models directly through weights \(W_i\). The goal of FL is to optimize the collective average of each client’s local loss function, formulated as:
\begin{align}
\min_{\{W_i\}_{i=1}^{n}} \sum_{i=1}^{n}\frac{|D_i|}{\sum_{j}|D_j|} L(X_i, Y_i, W_i),    
\end{align}
where \(n\) represents the number of clients participating in the training, $|\cdot|$ denote the cardinality of a set, and \(L\) denotes the loss function. In this paper, we assume $|D_i|$ are all equal.
FedAvg algorithm comprises two steps that are repeated until the model converges or a set number of rounds is completed:

(i) Local training: In Round $t$, clients use their local data \(D_i\) to update the local model parameters to \(W_i^{(t)}\). This involves calculating the gradient of the loss function \(L\), typically through batches of data to perform multiple gradient descent steps.

(ii) Global aggregation: all clients periodically send their updated parameters \(W_i^{(t)}\) to the central server. The central server aggregates these parameters using a specified algorithm, commonly by averaging: \(W_{global}^{(t)} = \frac{1}{k} \sum_{i=1}^{k} W_i^{(t)}\).  The aggregated global model parameters \(W_{global}^{(t)}\) are then sent back to all participating clients.
Each client sets its local model parameters to \(W_{global}^{(t)}\). 




{\bf Client selection and load balancing.}
Since communication usually occurs through a limited spectrum, only a subset $k$ out of $n$ total clients can update their parameters during each global aggregation round. This necessitates a strategy for \emph{client selection}. 

We aim to achieve load balancing by equalizing the number of iterations between consecutive client selections, assuming each client is selected with equal probability $\frac{k}{n}$. To that end, we define the \emph{load metric} \(X\) as the random variable for the number of rounds between subsequent selections of a client. Assume $X$ follows the same distribution for all clients. We propose the following optimization problem: 
\begin{align}
  &\min_{\mathcal{S}} \operatorname{Var}[X], \label{eq:min_var}\\
  &\text{s.t. } P(S_i^{(t)}=1)=\frac{k}{n}, i \in \{1,\dots,n\}, t \in \{1,\dots,T\}.\label{eq:constraint}
\end{align}
Here, $\mathcal{S}$ represents the set of permissible client selection policies, and $T$ is the total number of communication rounds. In addition, \(S_i^{(t)}\) takes value \(1\) if client \(i\) is selected during the $t$-th round, and takes value \(0\) otherwise.

Despite that all clients share the same selection probability and average workload over all iterations, as required in constraint \eqref{eq:constraint}, $\operatorname{Var}[X]$ focuses on the workload dynamics even when a small number of iterations is considered.
Minimizing \(\operatorname{Var}[X]\) creates predictable update intervals and promotes equitable load distribution among all clients. 
Moreover, the staleness of local updates, represented by a large $X$, may slow down the convergence of the training process; and frequent updates from any single client, or a small $X$, may disproportionately affect the model. 
Thus, \eqref{eq:min_var} can improve learning convergence and accuracy by ensuring consistent data freshness, which is verified in our simulation.  

{\bf Relation to Age of Information.}
A relevant metric is AoI, defined as the number of rounds elapsed since the last time a client was selected.
The age $A_i$ for Client \(i\) in Round $t$ evolves as follows:
\begin{align}
A_i^{\left ( t+1 \right )} =\left ( A_i^{\left ( t \right )} +1 \right ) \left ( 1-S_i^{\left ( t \right )}  \right ),  S_i^{\left ( t \right ) }\in \left \{ 0,1 \right \},
\end{align}
where \(A_i^{(0)}=0\). Therefore, in each round, a client's age increases by one if it is not selected and it resets to zero if it is selected.
The load metric $X$ is equal to the peak age, or the age before a client is selected.


\vspace{-0.2cm}
\section{Client Selection}\label{sec:selection}
\vspace{-0.2cm}
This section investigates random selection and a decentralized policy based on a Markov chain, assesses their load metric variances, and explores how to optimize the Markov model to contribute to a balanced  federated learning environment.



    


\vspace{-0.2cm}
\subsection*{Random Selection}
\vspace{-0.2cm}
The random selection method \cite{li2019convergence} involves choosing \(k\) out of \(n\) clients uniformly at random in each round for federated training. This approach, modeled mathematically by a geometric distribution, assumes each client has an equal probability \(p = \frac{k}{n}\) of being selected per round. The selection dynamics are described by:
\begin{align}
P(X=x) = p(1-p)^{x-1},
\end{align}
where \(X\) is the number of rounds until a client is selected and $x \in \{1,2,\dots\}$. The mean and variance of \(X\) are:
\begin{align}
&E[X] = \frac{1}{p} = \frac{n}{k}, \\
&\operatorname{Var}[X] = \frac{1-p}{p^2} = \frac{n(n-k)}{k^2}.
\end{align}

\subsection*{Decentralized Decision-Making Based on Local State}
This policy utilizes a decentralized approach where each client independently decides whether to send its data based on its current Age of Information, ranging from \(0\) to \(m\), with \(m\) representing the maximum permissible age. 
We formulate a Markov chain whose states represent the age of the client.
The probability of sending data at State $j$ is \(p_j\), for \(j = 0, 1, \ldots, m\), and in this case the state transitions to State $0$. With probability $1-p_j$, a client does not send data at State \(j\). Accordingly, the state transitions to State \(j+1\) for $j=0,1,\dots,m-1$; and the state remains the same for $j=m$.
States, transitions, and transition probabilities are depicted in Figure \ref{fig:markov}.

For a fair comparison with the random selection method, we impose a condition that the average number of clients selected in each round at steady state is \(k\) out of \(n\). 

This approach allows each client to autonomously manage its participation based on its current AoI without requiring communication or coordination with the central server or other clients. 

\begin{figure}
    \centering
\scalebox{0.6}{
\begin{tikzpicture}[->, >=stealth', auto, semithick, node distance=3cm, font=\Large]
\tikzstyle{state}=[circle, draw, minimum size=1cm]

\node[state] (zero) {0};
\node[state] (one) [right of=zero] {1};
\node[state] (two) [right of=one] {$2$};
\node (dots) [right=1cm of two] {$\cdots$}; %
\node[state] (m-1) [right=1cm of dots] {$m-1$}; %
\node[state] (m) [right of=m-1] {$m$};

\path (zero) edge [loop above] node {$p_0$} (zero)
      (zero) edge [bend left=50] node [above] {$1-p_0$} (one)
      (one) edge [bend left=50] node [below] {$p_1$} (zero)
      (one) edge [bend left=50] node [above] {$1-p_1$} (two)
      (two) edge [bend left=50] node [below] {$p_2$} (zero)
      (m-1) edge [bend left=50] node [above] {$1-p_{m-1}$} (m)
      (m-1) edge [bend left=50] node [below] {$p_{m-1}$} (zero)
      (m) edge [loop above] node {$1-p_m$} (m)
      (m) edge [bend left=50] node [below] {$p_m$} (zero);

\end{tikzpicture}
} 
\caption{Markov chain with $m+1$ states.}
\label{fig:markov}
\end{figure}
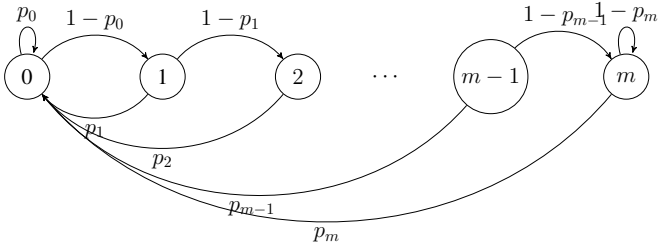

In the following, we find the optimal transition probabilities to minimize the variance of $X$ assuming that the client is at the steady state.
We first develop the expression of $\operatorname{Var}[X]$, then find the optimal Markov model with maximum permissible age $m=1$ in Theorem \ref{thm:optimal_m1}, and at last establish the optimal Markov model with a general $m$ in Theorem \ref{thm:optimal_general}.

Steady-state probabilities for the Markov chain are as follows:
\begin{align}
&\pi_0= \sum_{i=0}^{m} \pi_i p_i=\frac{k}{n}\\
&\pi_i = \pi_{i-1} \cdot (1 - p_{i-1}) \quad i \in \{1,2,...,m-1\}
\end{align}
\begin{align}
&\pi_m= \pi_{m-1} \cdot (1 - p_{m-1}) +\pi_{m} \cdot (1 - p_{m}) \\
&\sum_{i=0}^{m} \pi_i = 1.
\end{align}
By solving these equations we obtain:
\begin{align}
    \pi_0&= \frac{k}{n}= \frac{1}{1+\sum_{i=0}^{m-2} \prod_{j=0}^{i}(1-p_j) + \frac{1}{p_m} \prod_{j=0}^{m-1}(1-p_j)}, \label{eq:pi_0}
\end{align}
for $i \in \{1,..,m-1\}$,
\begin{align}
\pi_i&= \frac{\prod_{j=0}^{i-1}(1-p_j)}{1+\sum_{i=0}^{m-2} \prod_{j=0}^{i}(1-p_j) + \frac{1}{p_m} \prod_{j=0}^{m-1}(1-p_j)}, 
\end{align}
\begin{align}
\pi_m&= \frac{\frac{1}{p_m}\prod_{j=0}^{m-1}(1-p_j)}{1+\sum_{i=0}^{m-2} \prod_{j=0}^{i}(1-p_j) + \frac{1}{p_m} \prod_{j=0}^{m-1}(1-p_j)}. 
\end{align}
Our objective is to compute and minimize the variance of 
$X$, which represents the number of rounds required for a specific client to be returned to State $0$. Noting that \(\operatorname{Var}[X] = E[X^2] - E[X]^2\), we will derive both terms in this expression. Let us define $E_{0}=E[X]$, and $E_i$ to be the average time it takes to go to State $0$ from State $i$. Then we have the following set of equations:
\begin{align}
    E_i &= 1+ (1-p_i)E_{i+1}, \quad i \in \{0,1,...,m-1\},\\
    E_m &= 1+ (1-p_m)E_m. \label{eq:Em}
\end{align}
Equation \eqref{eq:Em} implies that 
$E_m= \frac{1}{p_m}$.
By \eqref{eq:pi_0} and solving the above equations, we can calculate $E_0$ as:
\begin{align}
   E_0= 1+ \sum_{i=0}^{m-2} \prod_{j=0}^{i} (1-p_j) + \frac{1}{p_m} \prod_{i=0}^{m-1}(1-p_i) = \frac{n}{k}, \label{e0}
\end{align}
and for $k \in \{1...,m-1\}$ we have:
\begin{align} \label{ek}
       E_k = 1+ \sum_{i=k}^{m-2} \prod_{j=k}^{i} (1-p_j) + \frac{1}{p_m} \prod_{i=k}^{m-1}(1-p_i).
\end{align}
Similarly, we can write the equations for $E[X^2]$ as follows:
\begin{align} 
    E[X^2] &= 1+ (1-p_0)(2E_1+ E[X_{1}^2]),\label{8}\\
    E[X_{i}^2] &= 1+ (1-p_i)(2E_{i+1}+ E[X_{i+1}^2]),  i \in \{1,...,m-1\},\label{9}\\
    E[X_{m}^2] &= 1+ (1-p_m)(2E_{m}+ E[X_{m}^2])=\frac{2-p_m}{p_{m}^2}.\label{10}
\end{align} 
By solving equations \eqref{8}, \eqref{9}, and \eqref{10} and performing algebraic simplifications, we derive the expression for $E[X^2]$ as follows:
\begin{align} \label{var}
   & 1+ \sum_{i=0}^{m-2} \prod_{j=0}^{i}(1-p_j)+2\sum_{i=0}^{m-1}E_{i+1}\prod_{j=0}^{i}(1-p_j)\nonumber\\&+E[X_{m}^2]\prod_{i=0}^{m-1}(1-p_{i}).
\end{align}
Accordingly, we can write the expression of $\operatorname{Var}[X]$. 

\begin{theorem}\label{thm:optimal_m1}
If $m=1$, the variance of $X$ is given by $\operatorname{Var}[X] = \frac{(1 + p_0 - p_1)(1 - p_0)}{p_1^2}$. Considering this, the optimal values of $p_0, p_1$ depend on the relationship between $k$ and $\frac{n}{2}$. Specifically:
\begin{itemize}
    \item If $k \leq \frac{n}{2}$, then the optimal values are $[p_0^*, p_1^*] = [0, \frac{k}{n-k}]$, resulting in a variance of $\frac{(n-k)(n-2k)}{k^2}$.
     \item If $k \geq \frac{n}{2}$, then the optimal values are $[p_0^*, p_1^*] = [\frac{2k-n}{k}, 1]$, resulting in a variance of $\frac{(n-k)(2k-n)}{k^2}$.
\end{itemize}
\end{theorem}
\begin{proof}
We simplify the constraint \eqref{e0} to obtain:
\begin{align}
1+\frac{1-p_0}{p_1}=\frac{n}{k}.    \label{eq:constraint_m1}
\end{align}
By writing down equations for $\operatorname{Var}[X]$, we have:
\begin{align}
    \operatorname{Var}[X]&= 1+(1-p_0)\frac{2+p_1}{p_1^2} - \left(\frac{n}{k}\right)^2\\
    &= 1+ \frac{n}{k}-\left(\frac{n}{k}\right)^2+ 2\frac{1-p_0}{p_1^2}\label{24}\\
    &= 1+ \frac{n}{k}-\left(\frac{n}{k}\right)^2+ 2 \frac{\frac{n}{k}-1}{p_1}\label{25},
\end{align}
where in \eqref{24} and \eqref{25}, we use \eqref{eq:constraint_m1}.
In order to minimize $\operatorname{Var}[X]$, we should set $p_1$ to its maximum value within its feasible range given the constraint \eqref{eq:constraint_m1}. Since $0 \leq p_0 \leq 1$, constraint \eqref{eq:constraint_m1} implies $\frac{1}{p_1}\geq \frac{n}{k}-1$, or $p_1 \leq \frac{k}{n-k}$. When $k \leq \frac{n}{2}$, $\frac{k}{n-k} \leq 1$ and it becomes the maximum feasible value for $p_1^*$. Consequently, we can find $p_0^*=0$ from $1+\frac{1-p_0}{p_1}=\frac{n}{k}$. If $k \geq \frac{n}{2}$, then the maximum feasible value for $p_1^*$ becomes $1$ and then we can find $p_0^*= \frac{2k-n}{k}$. 
\end{proof}
In both scenarios of Theorem \ref{thm:optimal_m1}, the variance of the Markov model is less than $\frac{n(n-k)}{k^2}$, which is the variance achieved by random selection.

\begin{theorem}\label{thm:optimal_general}
Consider an arbitrary $m \ge 1$.
\begin{itemize}
    \item When $m \le \lfloor \frac{n}{k} \rfloor-1 $, the optimal values of \(p_i\) for minimizing the variance are $[p_0^*, p_1^*, \dots, p_{m-1}^*, p_m^*] = [0, 0, \dots, 0, \frac{1}{\frac{n}{k} - m}]$, and the minimum variance is $(\frac{n}{k}-m)(\frac{n}{k}-(m+1))$.
    \item When $m \ge \lfloor \frac{n}{k} \rfloor$, setting $i = \lfloor \frac{n}{k} \rfloor$,
    the optimal values are $[p_0^*, p_1^*, \dots, p_{i-2}^*, p_{i-1}^*, p_{i}^*, \dots, p_m^*] = [0, 0, \dots, 0, i+1 -\frac{n}{k}, 1, \dots, 1]$, and the minimum variance is 
    $c(1-c)$, where $c = \frac{n}{k}-\lfloor \frac{n}{k} \rfloor$. 
\end{itemize}
\end{theorem}
\begin{proof}
    Based on \eqref{var} and \eqref{10}, $\operatorname{Var}[X]$ can be calculated as:
    \begin{align}
    &1+ \sum_{i=0}^{m-2} \prod_{j=0}^{i}(1-p_j)+2\sum_{i=0}^{m-1}E_{i+1}\prod_{j=0}^{i}(1-p_j)\\&+\frac{2-p_m}{p_m^2}\prod_{i=0}^{m-1}(1-p_{i}) -E_0^2.
    \end{align}
Using the condition that $E_0=\frac{n}{k}$, we can simplify this expression as:
\begin{align} \label{var_exp}
    \operatorname{Var}[X]&= \frac{n}{k}-\left(\frac{n}{k}\right)^2  +2\sum_{i=0}^{m-2}E_{i+1}\prod_{j=0}^{i}(1-p_j)\nonumber\\
    &+\frac{1}{p_m}\prod_{i=0}^{m-1}(1-p_{i}) +\frac{2-p_m}{p_m^2}\prod_{i=0}^{m-1}(1-p_{i}).
\end{align}
Based on \eqref{ek}, we have:
\begin{align}
  &E_{i+1} \prod_{j=0}^{i}(1-p_j) \\
  =& \prod_{j=0}^{i}(1-p_j) +\sum_{t=i+1}^{m-2} \prod_{j=0}^{t} (1-p_j) + \frac{1}{p_m} \prod_{t=0}^{m-1}(1-p_t) \\
  =& \frac{n}{k}-1 +\prod_{j=0}^{i} (1-p_j) - \sum_{t=0}^{i} \prod_{j=0}^{t} (1-p_j).
\end{align}
Therefore we can further simplify the expression in \eqref{var_exp} as:
\begin{align}
    \operatorname{Var}[X]&= \frac{n}{k}- \left(\frac{n}{k}\right)^2 +\frac{2}{p_m^2}\prod_{i=0}^{m-1}(1-p_{i})\nonumber\\
    &+ 2 \sum_{i=0}^{m-2} \left(\frac{n}{k}-1 - \sum_{t=0}^{i-1} \prod_{j=0}^{t} (1-p_j)\right) \\
    &=\left(2(m-1)+\frac{n}{k}\right)\left(1-\frac{n}{k}\right)+\frac{2}{p_m^2}\prod_{i=0}^{m-1}(1-p_{i})\nonumber\\
    &-2\sum_{i=0}^{m-2}  \sum_{t=0}^{i-1} \prod_{j=0}^{t} (1-p_j)\\
    &=\left(2(m-1)+\frac{n}{k}\right)\left(1-\frac{n}{k}\right)\nonumber\\
    &+\frac{2}{p_m}\left(\frac{n}{k}-1-\sum_{i=0}^{m-2} \prod_{j=0}^{i} (1-p_j)\right)\nonumber\\&-2\sum_{i=0}^{m-2} \sum_{t=0}^{i-1} \prod_{j=0}^{t} (1-p_j)\label{34}.
\end{align}
From \eqref{34}, it is obvious that for $\operatorname{Var}[X]$ to be minimized, assuming other parameters are fixed, $p_i$ for $i \in \{0,1,...,m-1\}$ should be minimum in its feasible range where the condition in \eqref{e0} is satisfied. Similarly, by condition \eqref{e0}, the second term in \eqref{34} is positive, and $p_m$ should be maximum when other parameters are fixed. 

When $m \le \lfloor \frac{n}{k}\rfloor-1$, we show that the above minimum and maximum values can be achieved at the same time. From \eqref{e0}, we find the following upper bound for $p_m$ and it happens when all $p_i$ values are $0$:
\begin{align}
    \frac{n}{k} \leq 1+ m-1 + \frac{1}{p_m} 
    \implies
    p_m \leq \frac{1}{\frac{n}{k}-m}.
\end{align}
Consequently, by choosing $p_i^*=0$ for $i \in \{0,1,...,m-1\}$ and $p_m^*=\frac{1}{\frac{n}{k}-m}$, we achieve the minimum of $\operatorname{Var}[X]$ while satisfying the constraint in \eqref{e0}.

Next consider $m \ge \lfloor \frac{n}{k}\rfloor$.
By defining $y_t  \triangleq \prod_{j=0}^t (1-p_i)$ we can rewrite the expression of $\operatorname{Var}[X]$ and make further simplifications.  
Based on Equation \eqref{34}, we need to minimize:
\begin{align}
    \frac{2}{p_m}\left(\frac{n}{k}-1-\sum_{i=0}^{m-2} y_i\right)-2\sum_{i=0}^{m-2}\sum_{t=0}^{i-1}y_t. \label{eq:obj_y}
\end{align}
The constraint \eqref{e0} becomes
\begin{align}
    1+y_0+y_1+...+\frac{1}{q_m}y_{m-1}=\frac{n}{k}.\label{eq:constraint_y}
\end{align}
The constraint $0 \le p_i \le 1$, for all $i \in \{0,1,\dots,m-1\}$, becomes
\begin{align}
    0\leq y_{m-1} \leq y_{m-2} \leq ... \leq y_0 \leq 1.\label{eq:constraint_y2}
\end{align}
Increasing $p_m$ can only increase $y_i$'s due to \eqref{eq:constraint_y}, and hence reduce the objective function \eqref{eq:obj_y}. Thus, we choose $p_m^*$ to its maximum feasible value, $p_m^*=1$.
Now, minimizing \eqref{eq:obj_y} becomes maximizing
\begin{align}
     2\sum_{i=0}^{m-2} y_i +2\sum_{i=0}^{m-2}\sum_{t=0}^{i-1}y_t
    =2\sum_{i=0}^{m-2} ((m-1-i)y_i),\label{eq:33}
\end{align}
subject to $1+y_0+y_1+...+y_{m-1}=\frac{n}{k}$, and \eqref{eq:constraint_y2}. 
This is a linear programming optimization problem, where we start with $y_{0}$ and set it to the maximum value possible, and then continue with the rest of $y_i$'s. Let $i = \lfloor \frac{n}{k} \rfloor$, we should set $y_{0},...y_{i-2}$ to $1$, $y_{i-1}=\frac{n}{k}-i$, and $y_i,...,y_{m-1}$ to $0$. This translates to $p_0^*,...,p_{i-2}^*=0$, $p_{i-1}^*=1-(\frac{n}{k}-i)$, and $p_{i}^*,...,p_{m-1}^*=1$. 

Finally, the minimum variance can be computed by plugging the optimal $p_0^*,\dots,p_m^*$ into the expression of $\operatorname{Var}[X]$.
\end{proof}

\begin{rem}
The optimal Markov model in Theorem \ref{thm:optimal_general} resembles the oldest-age method, which selects $k$ clients with the highest age in each iteration.  
The general Markov model, although not fully utilized in our current implementation, holds significant potential due to its flexibility. One can adaptively choose the transition probabilities in the Markov model based on real-time performance, system conditions, and client dropout probabilities.
For example, one can set $p_i > 0$ for all $i\in \{0,1,\dots,m\}$, and a client can be selected even if it does not have the oldest age, increasing the chance of getting an update from the client before it drops out of the network.
\end{rem}
\begin{rem}
For fixed $\frac{n}{k}$, as we increase $m$, the optimal variance of $X$ gets smaller, and the optimal $\operatorname{Var}[X]$ reaches the minimum for all $m \ge \lfloor \frac{n}{k} \rfloor$. For any $m \le \lfloor \frac{n}{k} \rfloor -1$, the optimal $\operatorname{Var}[X]$ is given by $(\frac{n}{k}-m)(\frac{n}{k}-(m+1))$, which is smaller than the variance of the random selection policy, $\frac{n}{k}(\frac{n}{k}-1)$. Hence, for all choices of $m$, the optimal Markov model provides better load balancing than random selection.


\end{rem}

\section{Simulation Results}\label{sec:simulation}

In this section, we present the model's convergence and accuracy in scenarios where only a partial subset of clients $(15\%)$ participate in each communication round. We compare random selection with the proposed decentralized client selection policy based on the optimal Markov model. The experiments are conducted on the MNIST, CIFAR-10, and CIFAR-100 datasets. In our simulations, we utilize a convolutional neural network (CNN) as described in \cite{mcmahan2017communication} to classify samples from the datasets. We employ the stochastic gradient descent (SGD) algorithm as the local optimizer, standardizing the batch size at 50 during the local training phase. The local training epochs are set to 5 per round, with an initial learning rate of 0.1 and a learning rate decay of 0.998. In our tests, we configure the parameters as follows: the number of devices \(n=100\), the number of selected devices in a round \(k=15\), and for the Markov model the largest age \(m=10\). 

As illustrated in Figure \ref{fig1}, for the CIFAR-10 dataset, our proposed Markov model for client selection achieves faster convergence. Specifically, it reaches the target accuracy of 80\% in just 240 communication rounds, compared to 265 rounds required by the random selection method, representing a 9.4\% improvement in convergence speed.
Additionally, as depicted in Figure \ref{fig2}, for the CIFAR-100 dataset, we observe that the Markov model reaches the target accuracy of $40\%$ in just $500$ communication rounds, compared to over $600$ rounds required by the random selection method. This results in an improvement of more than $20\%$ in convergence speed.

\begin{figure}
\centering
\includegraphics[width=0.37\textwidth]{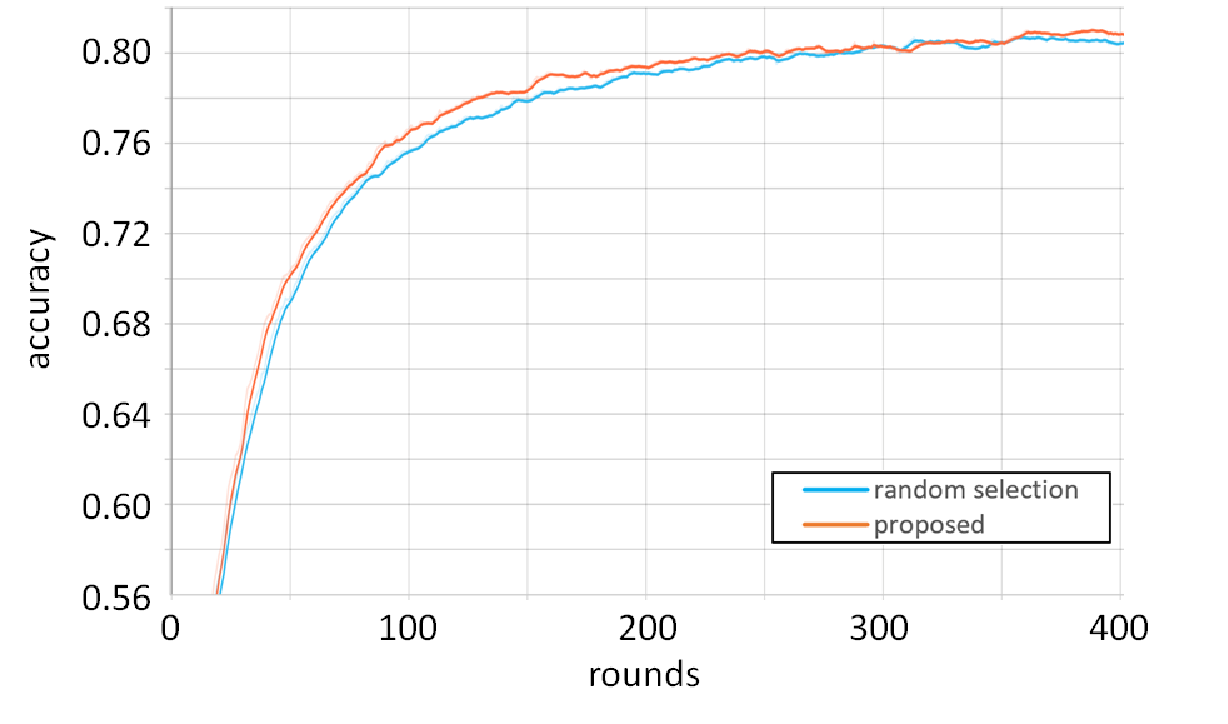} 
\caption{Comparison of accuracy between our proposed method (orange) and the random client selection method (blue) on the CIFAR-10 dataset with IID data distribution. The simulation parameters are $n=100$, $k=15$, and $m=10$.}
\label{fig1}
\end{figure}

\begin{figure}
\centering
\includegraphics[width=0.4\textwidth]{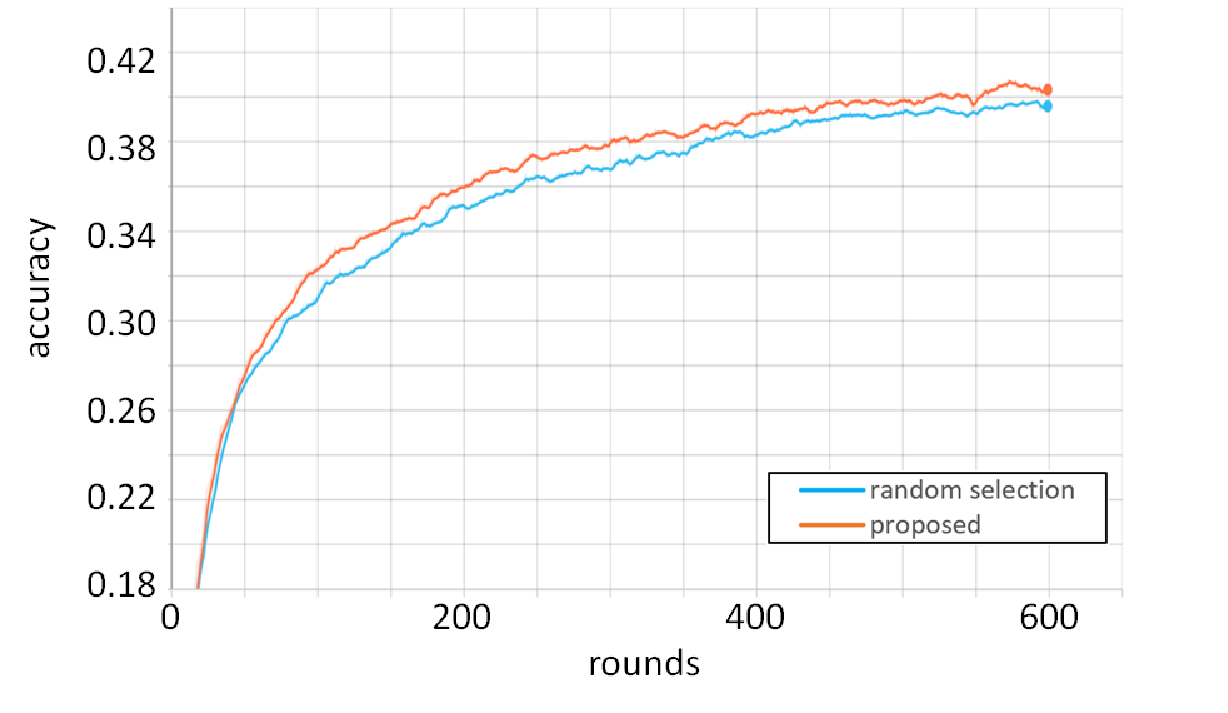} 
\caption{Comparison of accuracy between our proposed method (orange) and the random client selection method (blue) on the CIFAR-100 dataset with IID data distribution. The simulation parameters are $n=100$, $k=15$, and $m=10$.}
\label{fig2}
\end{figure}

In Figure \ref{fig3}, we present the performance analysis for the MNIST dataset under both IID and non-IID client data distributions. In the IID setting, our model achieves the target accuracy of $97\%$ in just $37$ communication rounds, which is faster than the $41$ rounds required by the random selection method, marking a $9.7\%$ increase in convergence speed. For the non-IID scenario, where client data follows a Dirichlet distribution with a parameter of $0.6$ \cite{yurochkin2019bayesian}, our model reaches the same target accuracy in $56$ rounds, compared to $64$ rounds with random selection, reflecting an improvement of $12.5\%$ in convergence speed.


\begin{figure}
    \centering
    \includegraphics[width=0.4\textwidth]{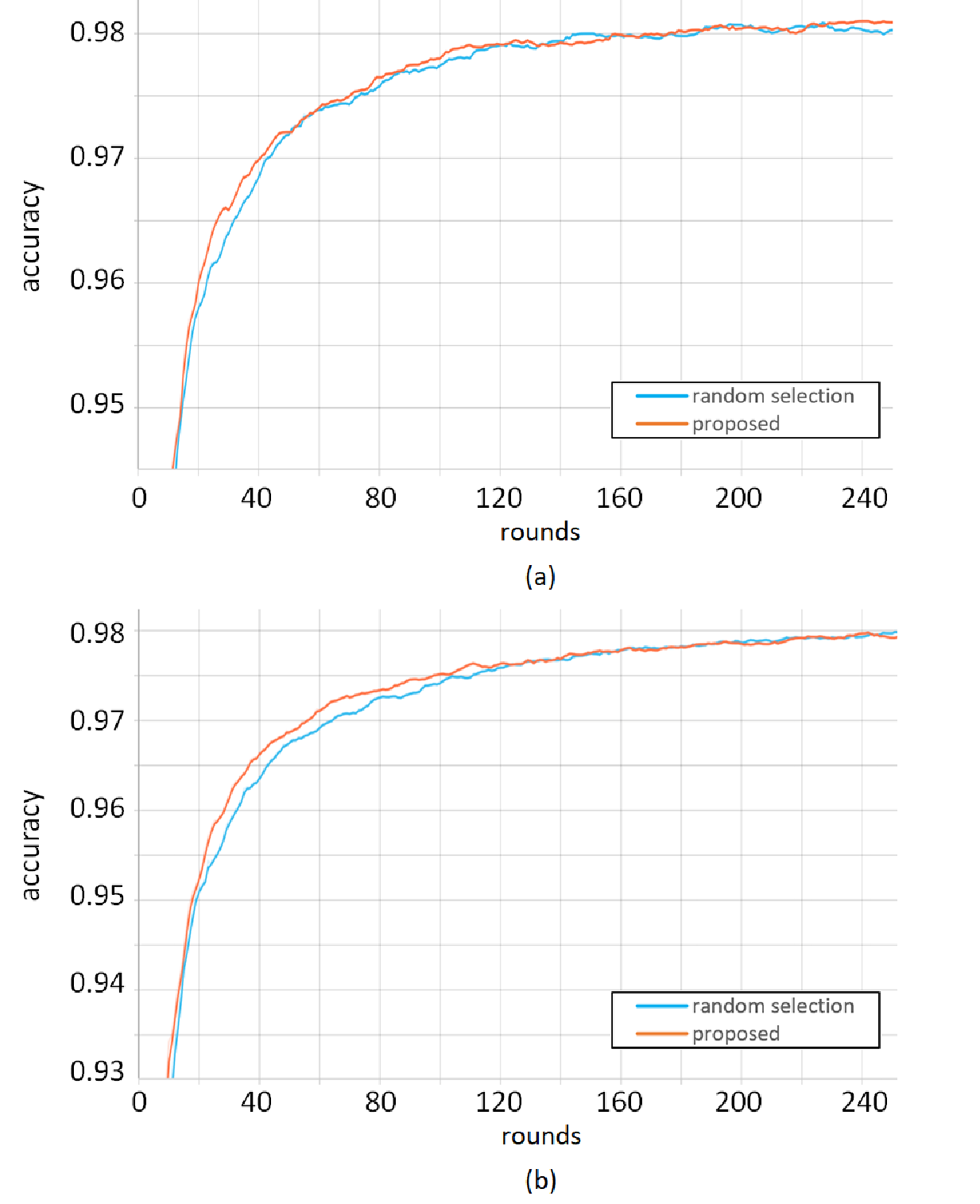}
    \caption{Comparison of accuracy between our proposed method (orange) and the random client selection method (blue) on the MNIST dataset with IID (top) and non-IID (bottom) data distribution. The simulation parameters are $n=100$, $k=15$, and $m=10$.}
    \label{fig3}
\end{figure}

\section{Conclusion}\label{sec:conclusion}
In conclusion, our research suggests that the decentralized client selection policy based on the Markov chain can lead to load balancing and efficient training in federated learning.

Future investigations should focus on a dynamic Markov chain policy for federated learning, adapting to the ongoing changes in network and client activity. By dynamically adjusting transition probabilities, we could enhance control over the Age of Information and improve responsiveness, thereby optimizing the trade-off between equitable client participation and overall system performance. 
\printbibliography
\end{document}